\documentclass[dvipsnames, table, 11pt]{article}

\usepackage{xcolor}
\usepackage[final]{acl}
\usepackage{amsthm}
\usepackage{amsmath}
\usepackage{geometry}
\usepackage{array}
\usepackage{enumitem}
\usepackage{booktabs}
\usepackage{graphicx}
\usepackage{amssymb}
\usepackage{wrapfig,lipsum,booktabs}
\usepackage{soul}
\usepackage{subcaption}
\usepackage{tikz}
\usepackage{thmtools}
\usepackage{soul}

\newtheorem{theorem}{Theorem}
\newtheorem{corollary}{Corollary}[theorem]

\usepackage[most]{tcolorbox}

\colorlet{LightGray}{White!98!Periwinkle}
\definecolor{LightCyan}{rgb}{0.75,1,1}

\declaretheoremstyle[
    name=Hypothesis,
]{thmsty}
\declaretheorem[style=thmsty]{hypothesis}

\tcolorboxenvironment{hypothesis}{enhanced jigsaw,colback=LightGray,drop shadow,boxrule=0.9pt, boxsep=0.1pt,left=4pt,right=4pt,top=4pt,bottom=4pt}

\usetikzlibrary{shapes,fit,backgrounds}

\usepackage{hyperref}
\usepackage{url}
\usepackage{subcaption}
\usepackage[english]{babel}

\def\mystrut(#1,#2){\vrule height #1pt depth #2pt width 0pt}   

\definecolor{newgreen}{RGB}{128,187,116}
\definecolor{neworange}{RGB}{198,154,89}
\definecolor{newblue}{RGB}{125,206,219}
\definecolor{matplotlibblue}{RGB}{0, 0, 255}
\definecolor{matplotliborange}{HTML}{ff7f0e}
\definecolor{matplotlibred}{HTML}{d62728}
\definecolor{matplotlibgreen}{HTML}{2ca02c}
\newcommand{\orangetext}[1]
{\colorbox{matplotliborange!80}{\mystrut(.5, .5) #1}}
\newcommand{\bluetext}[1]{\colorbox{newblue}{\mystrut(.5, .5) #1}}
\newcommand{\greentext}[1]{\colorbox{newgreen}{\mystrut(.5, .5) #1}}
\usepackage{times}
\usepackage{latexsym}

\usepackage[T1]{fontenc}

\usepackage[utf8]{inputenc}

\usepackage{microtype}

\usepackage{inconsolata}

\usepackage{graphicx}

\newcommand{\name}{\textsc{Cooldown}}

\title{\textit{Upsample} or \textit{Upweight?} \\
Balanced Training on Heavily Imbalanced Datasets}

\author{Tianjian Li, Haoran Xu, Weiting Tan \\ \textbf{Kenton Murray, Daniel Khashabi} \\
Center for Language and Speech Processing \\
Johns Hopkins University \\
\texttt{tli104@jhu.edu}}

\begin{document}
\maketitle
\begin{abstract}
Data abundance across different domains exhibits a long-tailed distribution: few domains have abundant data, while most face data scarcity. Our work focuses on a multilingual setting, where available data is heavily skewed towards high-resource languages. Two common strategies to address this disparity are upsampling low-resource data (\textit{Temperature Sampling}) and upweighting low-resource loss (\textit{Scalarization}). These methods are often assumed to be equivalent, but this equivalence has not been rigorously established, prompting our investigation.

Through theoretical and empirical analysis, we identify when these two methods are equivalent and when they diverge. We prove that they are equivalent under \emph{full} gradient descent but differ under \emph{stochastic} gradient descent due to differences in gradient variance. Specifically, \textit{Temperature Sampling} exhibits lower variance in gradient estimation compared to \textit{Scalarization}, leading to faster convergence but a higher risk of overfitting. Based on these insights, we propose \name{}, a strategy that starts by heavily upsampling low-resource languages to accelerate convergence and gradually reduces the upsampling to prevent overfitting—achieving the best of both worlds. Our method competes effectively with existing data re-weighting techniques while offering computational efficiency.
\end{abstract}

\setlength{\fboxsep}{1pt}

\section{Introduction}


Information on the internet ranges from common knowledge, such as famous landmarks, to rare details, such as local folklore and specialized scientific theories. Data availability across different domains is often long-tailed \citep{10.1145/3357713.3384290, 10.5555/3495724.3495966, pmlr-v202-kandpal23a}, where very few domains have abundant data. However, the standard language model training objective treats each training instance equally, putting no emphasis on domains that suffer from data scarcity. This heavy mismatch in dataset sizes creates substantial challenges in training language models to be competent in all domains.

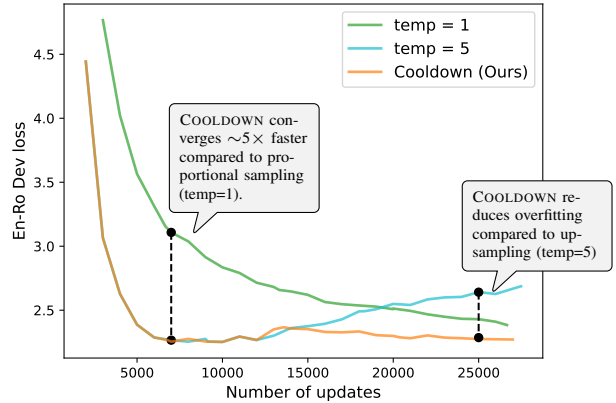
\begin{figure}[t]
    \centering
    \includegraphics[scale=0.5,trim=0.6cm 0.23cm 0.3cm 1.3cm,clip=true]{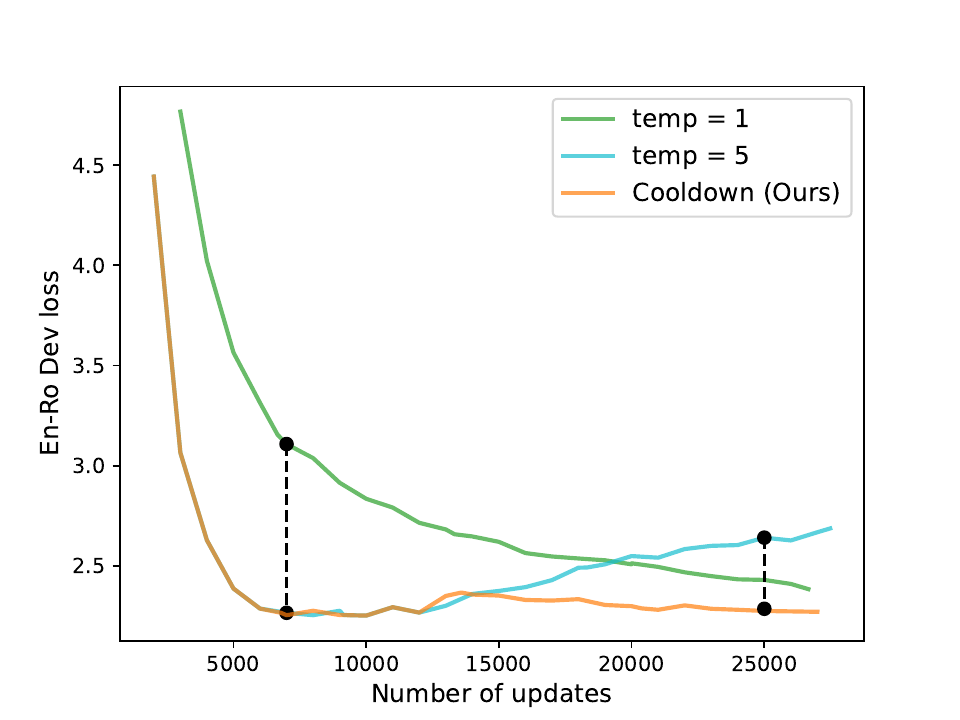}
        \begin{tikzpicture}[overlay]
        \node[font=\tiny, text width=1.7cm, yshift=2.8cm, xshift=3.15cm] (first) {\name~reduces overfitting compared to upsampling (temp=5)};
        \begin{pgfonlayer}{background}
            \node[fit=(first),inner sep=0mm,rectangle callout,
            callout relative pointer={(-0.2,-0.3)},
            rounded corners=2pt,draw,fill=gray!10,overlay] {};
        \end{pgfonlayer}
    \end{tikzpicture}
    \begin{tikzpicture}[overlay]
        \node[font=\tiny, text width=1.7cm, yshift=3.7cm, xshift=-0.8cm] (first) {\name~converges $\sim$5$\times$ faster compared to proportional sampling (temp=1).};
        \begin{pgfonlayer}{background}
            \node[fit=(first),inner sep=0mm,rectangle callout,
            callout relative pointer={(-0.2,-0.3)},
            rounded corners=2pt,draw,fill=gray!10,overlay] {};
        \end{pgfonlayer}
    \end{tikzpicture}
    \caption{Validation loss by training iteration of a low-resource language pair (\texttt{En-Ro}) in multilingual machine translation. \greentext{Proportional sampling} leads to underfitting the low-resource direction. Using a \bluetext{high temperature} (oversampling LRLs) leads to overfitting the low-resource direction. Employing a high temperature at the beginning and then decreasing the temperature (\orangetext{\name})~gets the advantage of fast convergence without overfitting. }
    \label{fig:teaser}
\end{figure}

This work focuses on language modeling on a natural divide of domains with a heavy mismatch: different languages in multilingual language modeling.  
Multilingual language models are often trained on corpora with an overwhelming amount of English and other high-resource languages (HRLs) and tiny amounts of data for low-resource languages (LRLs) \citep{koehn-knowles-2017-six, conneau-etal-2020-unsupervised, xue-etal-2021-mt5}. For example, the multilingual C4 corpus \citep{xue-etal-2021-mt5} contains 2733 billion English tokens but only 1 billion Swahili tokens. Uniformly sampling from the combined dataset would result in the language model optimized heavily towards performance on HRLs (e.g. English), sacrificing performance on LRLs (e.g. Swahili).

Two methods are often employed to address domain mismatches: \textbf{Scalarization} and \textbf{Temperature Sampling}. Scalarization adjusts the losses for individual domains by re-weighting them under uniform sampling \citep{zhou-etal-2021-distributionally, choi2023order}. In this case, we assign a larger weight to LRLs to emphasize their importance. Temperature Sampling weights each training instance uniformly and handles the mismatch by over-sampling LRLs and/or down-sampling HRLs \citep{aharoni-etal-2019-massively, wang-etal-2020-balancing, chung2023unimax, xue-etal-2021-mt5}. Intuitively, Scalarization modifies the loss while Temperature Sampling modifies the dataset. 
Scalarization and Temperature Sampling are widely regarded as equivalent. \citet{choi2023order} denotes \textit{``we follow convention and implement Scalarization via proportional sampling''}. \citet{xie2023doremi} and \citet{fan2024doge} implement sampling probabilities by multiplying losses with per-domain re-normalized weights. The underlying assumption is that Temperature Sampling and Scalarization are equivalent, and we can use them interchangeably. However, to the best of our knowledge, this equivalence has not been rigorously established.

We closely investigate this assumed equivalency in theory (\S \ref{theoretical}). Specifically, we prove that although they are equivalent in \textit{full} gradient descent (Theorem \ref{theorem:1}), Temperature Sampling induces lower variance in the context of  \textit{stochastic} gradient descent (Theorem \ref{theorem:2}). Moreover, the variance induced by scalarization increases as the approximated temperature increases or the domain distribution's skewness increases (Theorem \ref{theorem:3}). Based on our theoretical results and connecting to the literature on lower variance between stochastic gradients accelerates convergence \cite{pmlr-v28-sutskever13, mccandlish2018empiricalmodellargebatchtraining}, we make the following hypothesis:
\vspace{-0.2cm}
\begin{hypothesis}
    \label{hypothesis1}
    Temperature Sampling converges much faster than Scalarization at higher temperatures or on heavily imbalanced domain distributions.
\end{hypothesis}
We empirically verify our hypothesis (\S\ref{section3}) and find that Temperature Sampling does converge faster but is more prone to overfitting. We identify that the temperature controls the speed of convergence and hence can be used as a control knob to adjust the convergence speed. We thus propose \name: to use a large temperature initially for fast convergence, then decrease the temperature to prevent overfitting to the LRLs. Figure \ref{fig:teaser} illustrates the effectiveness of \orangetext{\name}, which significantly accelerates convergence on the LRL due to a high temperature (aggressive upsampling of LRLs) at the beginning of training and reduces overfitting to the LRL due to the lowering the temperature during training. 

To sum up, our contribution is two-fold:
\begin{itemize}[leftmargin=*]
    \item We inspect Scalarization and Temperature Sampling both theoretically and empirically (\S 3). Contrary to existing work that uses them interchangeably, we found that Temperature Sampling converges faster due to a lower variance in stochastic gradient estimation.
    \item Motivated by our findings, we propose \name, a method to adjust the sampling temperature during training on unbalanced datasets. We show the effectiveness of \name~ in multilingual settings.
\end{itemize}

\section{Preliminaries}
\label{section2}
\subsection{Notations and Task Description} We consider a model trained on a collection of data $\mathcal{D} = \{x\}_{i=1}^N$ from $K$ domains $\mathcal{D} = \mathcal{D}_1 \cup \mathcal{D}_2 \cup  ... \cup \mathcal{D}_{K}$. Here, ``domain" refers to sources (Books, Wikipedia, code) for general language modeling or different languages (English, French, Swahili) in multilingual language modeling. The total training loss $\mathcal{J}(\mathcal{D})$ is the sum of the losses of each example $\mathcal{L}(x)$. 

$$\mathcal{J}(\mathcal{D}) = \sum_{x \in \mathcal{D}} \mathcal{L}(x).$$

\paragraph{Scalarization (S)} Naive aggregation often results in imbalanced performance across domains when high-resource domains dominate the aggregated loss. Scalarization solves this issue by assigning weights $\mathbf{w} = \{w_{i}\}_{i=1}^K$ to each domain and aggregates the weighted sum of individual losses:
$$ \mathcal{L}_{S}(\mathbf{w}) = \mathbb{E}_{x \in \mathcal{D}} \left[w_{f(x)} \mathcal{L}(x)\right],$$
where $f: \mathcal{D} \rightarrow [K]$ maps a training example to the index of its domain.
Scalarization balances the loss by assigning a higher weight to harder or low-resource domains.
\paragraph{Temperature Sampling (TS)} Instead of assigning weights to losses, we can also sample more frequently from the low-resource domain to achieve balanced training. Temperature sampling achieves this by adjusting the probabilities of selecting instances from different domains based on their sizes. The sampling probability vector $p$ of each domain is given by: 
$$ \forall i \in \{1, 2, ..., K\}: \; p(i; \tau) = \frac{|\mathcal{D}_i|^\frac{1}{\tau}}{\sum_{j=1}^K |\mathcal{D}_j|^\frac{1}{\tau}},$$
where $\tau$ is the sampling temperature, a hyperparameter controlling the sampling weights. $\tau=1$ means that we are sampling proportional to the sizes of each domain. As we increase $\tau$, we increase the sampling probability of low-resource domains. The loss for Temperature Sampling is:

$$
\mathcal{L}_{TS}(\tau) = \mathop{\mathbb{E}}_{\substack{k \sim p(\cdot ; \tau) \\ x \sim \mathcal{D}_k}} \Big[ \mathcal{L}(x) \Big]
$$
The common understanding is that Temperature Sampling is mathematically equivalent to Scalarization \citep{choi2023order, NEURIPS2022_580c4ec4}, and we can use them interchangeably \citep{choi2023order}. In the next subsection, we will formalize this statement and show that these two are mathematically equivalent in \textit{full} gradient descent.
We will then show that they are \underline{not} equivalent under \emph{stochastic} gradient descent.

\section{Temperature Sampling v.s. Scalarization}

\subsection{Theoretical Analysis}
\label{theoretical}
We formalize the equivalence of Scalarization and weighted sampling under full-gradient descent (\autoref{theorem:1}) and show that Scalarization induces a larger variance between the mini-batch losses (\autoref{theorem:2}). Furthermore, when using Scalarization to approximate Temperature Sampling, the variance increases as the temperature rises (\autoref{theorem:3}).


\begin{theorem}[Equivalency under Gradient Descent]
\label{theorem:1}
For any sampling temperature $\tau$, there exists a set of weights $\mathbf{w}_\tau = \{w_1, w_2, ..., w_K\}$ for the Scalarization loss such that this loss is equivalent to 
the Temperature Sampling loss, both computed based on the whole data $\mathcal{D}$.
\end{theorem}

\begin{proof}
    For all $i \in \{1, 2, 3,..., K\}$, let: 
    $$w_i = \frac{\sum_{j=1}^K |\mathcal{D}_j|}{|\mathcal{D}_i|} \cdot \frac{|\mathcal{D}_i|^\frac{1}{\tau}}{\sum_{j=1}^K |\mathcal{D}_j|^\frac{1}{\tau}} = \frac{p(i; \tau)}{p(i; 1)},$$ 
    \begin{align*}
    \mathcal{L}_{S}(\mathbf{w}_\tau) &= \mathop{\mathbb{E}_{x \sim \mathcal{D}}}\left[w_{f(x)}\mathcal{L}(x)\right] \\
    &= \sum_{x \in \mathcal{D}} \frac{p(f(x);1)}{|\mathcal{D}_{f(x)}|}w_{f(x)} \mathcal{L}(x) \\ &=  \sum_{i = 1}^K p(i; \tau)\sum_{x \in \mathcal{D}_i} \frac{1}{|\mathcal{D}_i|}\mathcal{L}(x) \\ &= \mathop{\mathbb{E}}_{i \sim p(\cdot; \tau)} \left[\sum_{x \in \mathcal{D}_i} \frac{1}{|\mathcal{D}_i|}\mathcal{L}(x)\right] \\ &=  \mathbb{E}_{i \sim p(\cdot;\tau)} [\mathbb{E}_{x \sim \mathcal{D}_i} [\mathcal{L}(x)]]=\mathcal{L}_{TS}(\tau).
    \end{align*}
\end{proof}

Intuitively, this suggests that in the context of \textit{full} gradient descent, the loss remains the same whether you multiply the loss of a single data point by 2 (S) or duplicate the data point (TS). We will then show that Scalarization induces a larger variance in \textit{stochastic} gradient estimation \citep{Robbins1951ASA} compared to Temperature Sampling.

\begin{corollary}
    For any $\tau$, let $\mathbf{w}_\tau$ be the set of weights such that $\mathcal{L}_{TS}(\tau) = \mathcal{L}_{S}(\mathbf{w}_\tau)$. Let $\nabla \mathcal{L}(x)$ be the gradient with respect to a single datapoint $x$. We denote the stochastic gradient under Scalarization as $\nabla \mathcal{L}_{S}(x;\mathbf{w}_\tau) = \{\nabla w_{f(x)} \mathcal{L}(x)| x\sim\mathcal{D}\}$, which is the gradient $\nabla w_{f(x)}\mathcal{L}(x)$ when a single sample $x$ is uniformly drawn from the dataset $\mathcal{D}$. Similarly, we denote the stochastic gradient under Temperature Sampling as $\nabla \mathcal{L}_{TS}(x; \tau) = \{ \nabla \mathcal{L}(x) | x \sim \mathcal{D}_i, i \sim p(i; \tau)\}$. Then both $\nabla \mathcal{L}_{S}(x; \mathbf{w}_\tau)$ and $\nabla \mathcal{L}_{TS}(x; \tau)$ are unbiased estimates of the total gradient.
\end{corollary}

\begin{proof}

By definition, we have 
\begin{align*}
\mathbb{E}[\nabla \mathcal{L}_s(x; \mathbf{w}_\tau)] &= \mathop{\mathbb{E}}_{x \sim \mathcal{D}} \left[w_{f(x)} \nabla \mathcal{L}(x)\right] \\
&= \mathbb{E}_{\mathcal{D}_i \sim p(\cdot;\tau)} [\mathbb{E}_{x \sim \mathcal{D}_i} [\nabla \mathcal{L}(x)]] \\
&=\mathbb{E}[\nabla \mathcal{L}_{TS}(x; \tau)]
\end{align*}

\end{proof}


\begin{theorem}[Scalarization induces larger variance under Stochastic Gradient Descent]
\label{theorem:2}
 Using the same notation in Corollary 1.1, we have $\mathrm{Var}(\nabla \mathcal{L}_{S}(x; \mathbf{w}_\tau)) \geq \mathrm{Var}(\nabla \mathcal{L}_{TS}(x; \tau))$.
\end{theorem} 

\noindent We defer the proof of Theorem 2 to Appendix \ref{proof_of_theorem}. \\


\paragraph{Implication of Theorem 2} Scalarization induces larger variance connects with the literature on variance reduction in stochastic gradient estimation \citep{pmlr-v28-sutskever13, KingBa15} accelerates the convergence of SGD \citep{Robbins1951ASA}. We thus hypothesize that Temperature Sampling will converge faster than Scalarization with the set of weights that makes them mathematically equivalent under \textit{full} gradient descent. 

\begin{theorem}[Scalarization induces larger variance when approximating higher temperatures]
\label{theorem:3}
The difference:
$$\Delta = \mathrm{Var}(\nabla \mathcal{L}_{S}(x; \mathbf{w}_\tau)) - \mathrm{Var}(\nabla \mathcal{L}_{TS}(x; \tau))$$ 
monotonically increases when $\tau \geq 1$.
\end{theorem} 

\noindent Figure \ref{fig:plot_distribution} illustrates the variance by increasing temperature and skewness of the distribution. We defer details on how we constructed Figure \ref{fig:plot_distribution} and the full proof of Theorem 3 to Appendix~\ref{proof_of_theorem_3}.\\


\begin{figure}[ht]
\vspace{-0.7cm}
  \begin{center}
    \includegraphics[scale=0.48,trim=0cm 0.8cm 0cm 0.3cm]{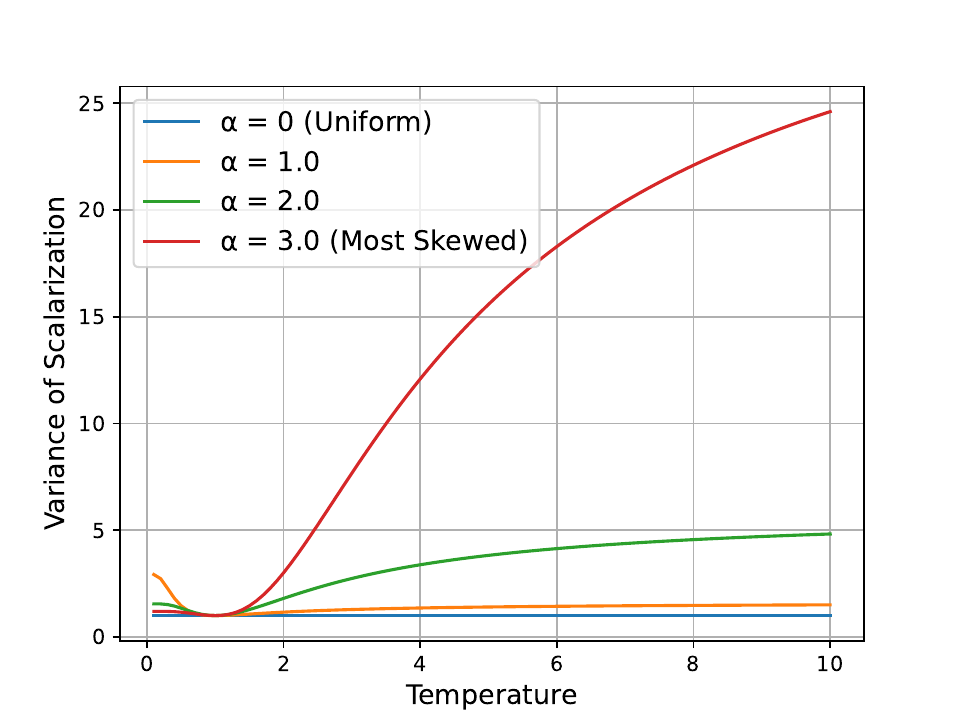}
  \end{center}
  \caption{Variance of Scalarization $\sum_i \frac{p(i; \tau)^2}{p(i; 1)}$ by sampling temperature $\tau$. A large temperature or a skewed distribution of $\mathcal{D}$ induces a much larger variance for Scalarization. Distributions $\mathcal{D}_i \propto \frac{1}{i^\alpha}$. See Appendix \ref{proof_of_theorem_3} for details of the experiment setup.}
  \vspace{-0.5cm}
  \label{fig:plot_distribution}
\end{figure}

\paragraph{Implication of Theorem 3} The fact that the induced variance of Scalarization increases as the approximated temperature increases implies that Temperature Sampling converges much faster than Scalarization at higher temperatures, which we empirically verify in the next section.

\subsection{Empirical Evidence}
\label{section3}

We directly validate our hypothesis that lower variance of Temperature Sampling accelerates convergence.\footnote{We apply the Adam \citep{KingBa15} optimizer with mini-batch gradient descent.} We train a multilingual machine translation model and vary the sampling temperature $\tau = \{2, 3, 5\}$. We then approximate Temperature Sampling by multiplying the Temperature Sampling probabilities by loss under proportional sampling ($\tau=1$). Specifically, we pair one high-resource direction in \texttt{En-}$\{$\texttt{Fr, Cs}$\}$ with the low-resource direction \texttt{En-Ro}. We report the statistics of the datasets we used in \S 3.2 at Table \ref{datasummary}, for all of our experiments, we learn a shared Byte-Pair Encoding tokenizer with a vocabulary size of 32k.\footnote{Additional experiments on \texttt{En-\{Zh, Ro\}} can be found in Appendix~\ref{zhroresults}.} 

\begin{table}[h]   
\centering 
\setlength\tabcolsep{12pt}
\begin{tabular}{@{}lcc@{}}
\toprule
Language Pairs      & Training & Validation \\ \midrule
WMT 15 En-Fr & 40M      & 4503       \\
WMT 22 En-Zh & 55M      & 3418       \\
WMT 22 En-Cs & 56M      & 2082       \\
WMT 16 En-Ro & 600K     & 1999       \\ \bottomrule
\end{tabular}
\caption{Dataset statistics for comparing Scalarization v.s. Temperature Sampling in \S \ref{section3}. First three language pairs \texttt{En-}$\{$\texttt{Fr,Zh,Cs}$\}$ are high-resource.}
\label{datasummary}
\vspace{-0.5cm}
\end{table}

\paragraph{Empirical Validation of Theorem 2} We first validate that Temperature Sampling has a lower variance in gradient estimation than scalarization, as predicted by Theorem 2. Figure~\ref{fig:gnorm distribution} illustrates the variance between mini-batch gradients for TS and scalarization, confirming that TS reduces gradient variance ($2.25 \rightarrow 0.62$).

\begin{figure}[ht]
  \begin{center}
    \includegraphics[scale=0.48,trim=0cm 0.5cm 0cm 0.9cm]{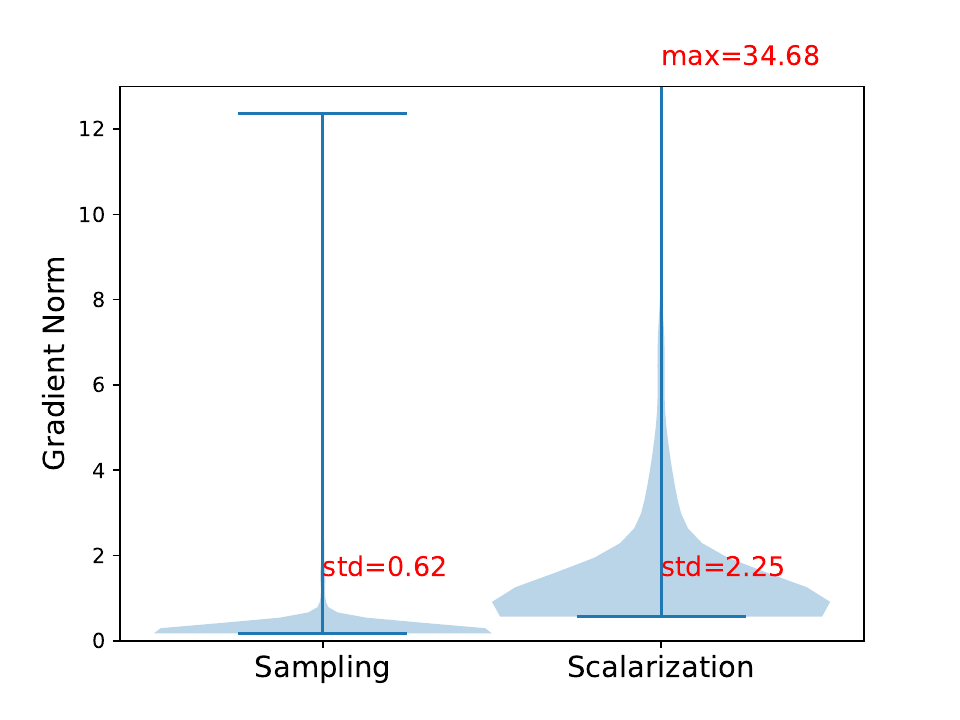}
  \end{center}
  \vspace{-5pt}
  \caption{The distribution of gradient norm between mini-batches on \texttt{En-\{Cs, Ro\}} for Temperature Sampling and Scalarization. \textbf{Scalarization induces a larger variance (2.25 $>$ 0.62) between mini-batch gradient norms compared to Temperature Sampling, as indicated by Theorem \ref{theorem:3}.}}
  \label{fig:gnorm distribution}
  \vspace{-0.2cm}
\end{figure}

\begin{figure*}

\centering
\includegraphics[scale=0.36]{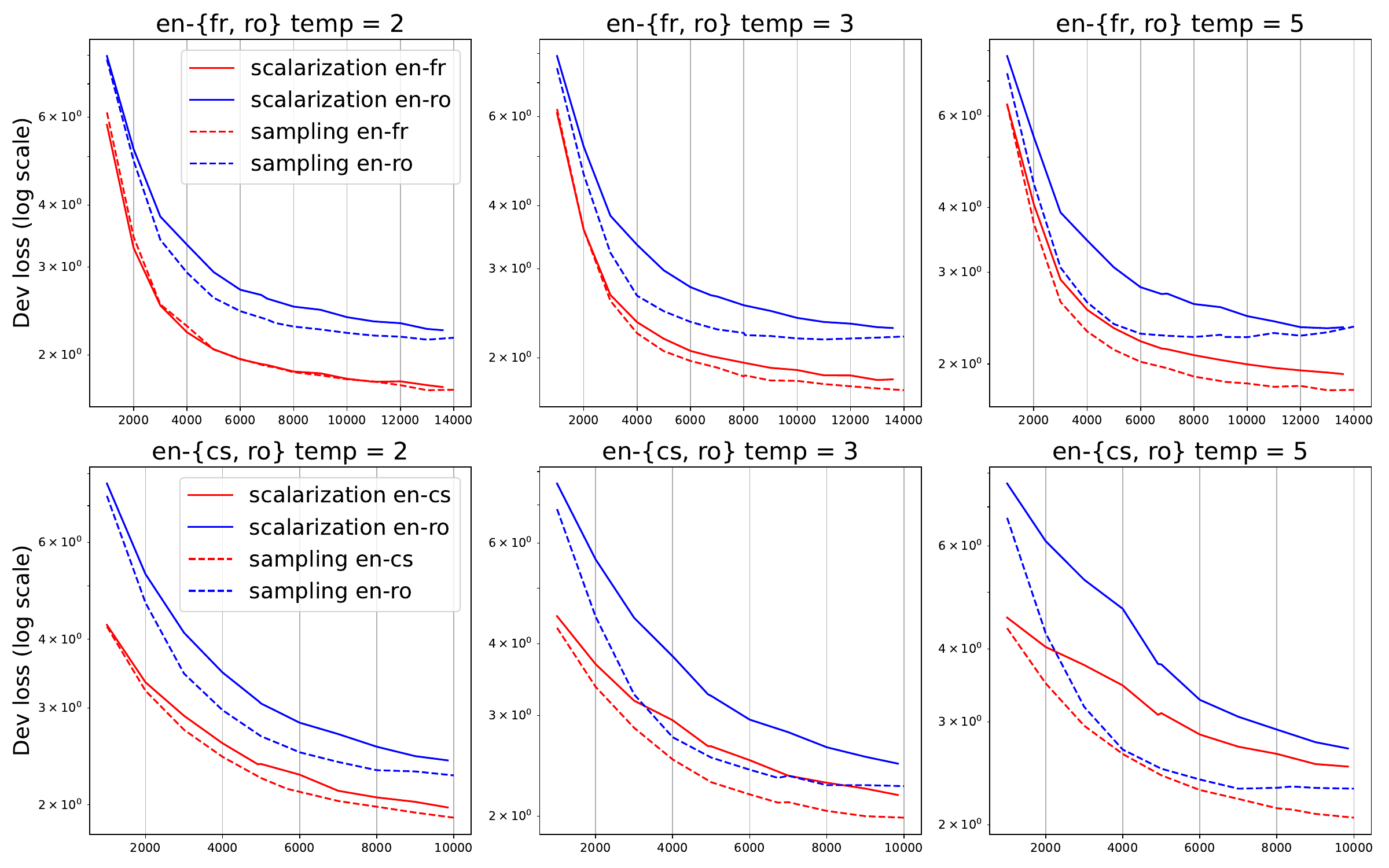}
    \caption{Validation loss by training iteration for \texttt{En-\{Cs, Ro\}} (first row) and \texttt{En-\{Fr, Ro\}} (second row). \textbf{Temperature Sampling (dashed) converges faster compared to Scalarization (solid), leading to better performance on both the \sethlcolor{matplotlibred}\hl{HRL} and the \sethlcolor{matplotlibblue!75}\hl{LRL}.}
    }
    \label{fig:s_vs_ts}
\end{figure*}

Next, we observe that this lower variance leads to faster convergence during training. Figure~\ref{fig:s_vs_ts} shows the validation loss curves over training iterations, where TS consistently converges faster than scalarization across all temperatures we experimented with. Notably, the larger the temperature, the greater the gap in convergence speed between TS and Scalarization. This suggests that when there is a significant mismatch in data sizes, using Temperature Sampling with a high temperature is beneficial. Intuitively, low-resource languages (LRLs) with very little data should be upsampled more aggressively to accelerate convergence. We summarize our findings below:

\paragraph{Large Temperature Sampling is prone to overfitting}
In our \texttt{En - \{Fr, Ro\}} experiments, we observed that the model overfits the low-resource direction (Ro) when using a large temperature ($\tau = 3, 5$). However, the high-resource direction has not yet converged when this overfitting occurs; therefore, continuing to train in both directions would lead to severe overfitting. This indicates that we need to pair strong regularization on the LRL (e.g. early stopping) with a large temperature, which motivates our temperature scheduling method \name, which uses a large temperature during the beginning to speed up training and then decreases the temperature to prevent overfitting on the LRL.

\paragraph{Temperature Sampling is equivalent to Scalarization given enough compute} We found that Temperature Sampling (dashed in Figure~\ref{fig:s_vs_ts}) always converges faster compared to weighting the losses of individual directions (solid in Figure~\ref{fig:s_vs_ts}), but they eventually converge to the same validation loss given enough training iterations, which corresponds to their equivalency under \textit{full} gradient descent (\autoref{theorem:1}). This means both Scalarization and Temperature Sampling can effectively balance multiple languages when given a large enough compute.

\section{\name: Balanced Training for Heavily Imbalanced Datasets}
\label{section:experiments}
Based on the theoretical analysis in \S \ref{theoretical} and the empirical results in \S \ref{section3}, we conclude that \textbf{Temperature Sampling with a large temperature converges faster than Scalarization with equivalent weights but is more prone to overfitting on the LRL when using a large temperature.} We thus hypothesize that we can employ a large temperature during the beginning of training to speed up convergence and then decrease the temperature to prevent overfitting on low-resource directions. 

\paragraph{Our proposed method:} We design a simple temperature scheduling method: \name~which starts with a high temperature (aggressive upsampling of LRLs) and then lowers the temperature to $\tau=1$ (proportional sampling) at a fixed iteration to prevent overfitting. We describe our experiment setup at \S\ref{subsec:setup}, report results at \S\ref{subsec:main:results}, and discuss our findings at \S\ref{subsec:findings:ablations}.

\subsection{Setup}
\label{subsec:setup}

\paragraph{Models, Datasets, and Hyper-parameters} We experiment on two setups: \textbf{multilingual machine translation} and \textbf{multilingual language modeling}, both suffering severe mismatch in dataset sizes. For our machine translation experiments, we use the standard encoder-decoder Transformer~\citep{Vaswani+2017} architecture implemented in fairseq \citep{ott2019fairseq}. We select 8 distinct languages from the opus-100 dataset \citep{zhang-etal-2020-improving} and train a one-to-many translation where the source language is English, we used a shared BPE tokenizer with 64k vocabulary. Detailed Languages and their respective sizes can be found in Table \ref{table:mt}. For our multilingual language modeling experiments, we use a decoder-only Transformer model from Huggingface \citep{wolf-etal-2020-transformers} and select 4 linguistically diverse languages with varying amounts of data from the mC4 \citep{xue-etal-2021-mt5} dataset. The statistics are in Table \ref{table:mc4_stats}. We used the mT5 tokenizer \cite{xue-etal-2021-mt5} for our experiments on mC4.

\begin{table}[h]
\centering
\setlength\tabcolsep{12pt}
\begin{tabular}{@{}lcc@{}}
\toprule
Languages    & Tokens & \multicolumn{1}{c}{\begin{tabular}[c]{@{}c@{}} mT5 weights\\ $\tau$ = 3.33\end{tabular}} \\ \midrule
EN - English & 2733B  & 5.67\%                                                                                       \\
IT - Italian & 162B   & 2.43\%                                                                                       \\
ZH - Chinese & 39B    & 1.67\%                                                                                       \\
SW - Swahili & 1B     & 0.5\%                                                                                        \\ \bottomrule
\end{tabular}
\caption{Dataset statistics of our selected subset of C4. mT5 \citep{xue-etal-2021-mt5} effectively uses an sampling temperature of $\tau = 3.33$ to oversample LRLs.}
\label{table:mc4_stats}
\end{table}

 For our machine translation experiments, we use $\tau=5$ for the first 30k of training iterations and $\tau=1$ for the 
second 30k. For our language modeling experiments, we use $\tau = 5$ for the first 50k training iterations and $\tau = 1$ for the second 50k. Detailed hyper-parameters are in Appendix \ref{appendix:hparams}. 

\paragraph{Baselines} We experiment on baselines that apply a fixed temperature throughout training (static temperature) and baselines that adjust the temperature during training (dynamic temperature). For static temperature, we vary the sampling temperature in $\tau = 1$ (Proportional Sampling), $\tau = 5$ and $\tau = 100$ ($\sim$Uniform Sampling). For dynamic temperature, we compare with \textbf{Unimax} \citep{chung2023unimax}, which first heavily upsamples low-resource languages and removes them after the low-resource dataset has been seen by the model for a fixed amount of repetitions, and \textbf{Order Matters} \citep{choi2023order} which first only trains on high-resource languages, and only adds in low-resource languages to the end of training. Additionally, we include the results of \textbf{DoReMi} \citep{xie2023doremi}, which trains small proxy models that minimize the loss of a worse performing set of domains iteratively to find optimal sampling probabilities of each domain for training a large model.

\subsection{Main Results}
\label{subsec:main:results}

\begin{table*}[ht]
\centering
\setlength\tabcolsep{5pt} 
\small
\begin{tabular}{@{}lccccccccc|cc@{}}
\toprule
\texttt{en}-$\{\}$             & \texttt{es}   & \texttt{fa}   & \texttt{ga}   & \texttt{gl}   & \texttt{ha}   & \texttt{hi}   & \texttt{it}   & \texttt{kk}   & \texttt{ug}   & HRLs  & LRLs           \\ 
\# of Parallel Sentences &1M	&1M	&294K&	519K	&102K	&538K&	1M	&83K	&76K & $>$1M & $<$500K \\ \midrule
\multicolumn{12}{@{}l@{}}{\textbf{Static Temperature Sampling}} \\ \midrule
$\tau=1$ (Proportional)             & \textbf{38.9} & \sethlcolor{green!20}\hl{13.1} & \sethlcolor{red!20}\hl{58}   & \textbf{28.9} & \sethlcolor{red!20}\hl{41.9} & \sethlcolor{green!20}\hl{17.1} & \textbf{32.8} & \sethlcolor{red!20}\hl{22.4} & \sethlcolor{green!20}\hl{10.8} & \textbf{28.3}  & \sethlcolor{red!20}\hl{30.4}  \\
$\tau=5$            & \sethlcolor{red!20}\hl{36.9} & 12.2 & \textbf{60.9} & \sethlcolor{green!20}\hl{28.3} & 46.3 & \sethlcolor{green!20}\hl{16.8} & \sethlcolor{red!20}\hl{30.7} & \sethlcolor{green!20}\hl{26.8} & \sethlcolor{red!20}\hl{9.4}  & \sethlcolor{red!20}\hl{26.6}  & \textbf{32.4}        \\ 
$\tau = 100$ ($\sim$Uniform) & \sethlcolor{red!20}\hl{36.1} & 12.3 & 59.6 &   \sethlcolor{red!20}\hl{27.3} & \sethlcolor{green!20}\hl{46.4} & \textbf{17.3} & \sethlcolor{red!20}\hl{30.3} & \textbf{27.6} & \sethlcolor{red!20}{9.1} & \sethlcolor{red!20}\hl{26.2} &  \sethlcolor{green!20}\hl{31.4} \\ \midrule
\multicolumn{12}{@{}l@{}}{\textbf{Dynamic Temperature Sampling}} \\ \midrule
\textrm{Unimax \citep{chung2023unimax}}         & 37.1 & 12.2 & \sethlcolor{green!20}\hl{60.7} & \sethlcolor{green!20}\hl{28.4} & \sethlcolor{red!20}\hl{45.9} & \sethlcolor{green!20}\hl{16.5} & \sethlcolor{red!20}\hl{30.8} & 26.2 & \sethlcolor{red!20}\hl{9.4}  & \sethlcolor{red!20}\hl{26.7}               & \sethlcolor{green!20}\hl{32.1}  \\
\textrm{Order Matters \citep{choi2023order}}          & 37.1 & 12.2 & \textbf{60.9} & \sethlcolor{green!20}\hl{28.2} & 46.1 & \sethlcolor{green!20}\hl{16.7} & \sethlcolor{red!20}\hl{30.8} & \sethlcolor{green!20}\hl{26.7} & \sethlcolor{red!20}\hl{9.4}  & \sethlcolor{red!20}\hl{26.7}      & \sethlcolor{green!20}\hl{32.3}         \\ 
\sethlcolor{LightCyan}\hl{\name~(Ours)} & \sethlcolor{green!20}\hl{38.7} & \textbf{13.2} & \sethlcolor{green!20}\hl{60.1}   & \sethlcolor{green!20}\hl{28.7} & \textbf{47.4} & \textbf{17.3}   & \sethlcolor{green!20}\hl{32.2} & \sethlcolor{red!20}\hl{26} & \textbf{11} &  \sethlcolor{green!20}\hl{28.1}       & \textbf{32.4}        \\ \midrule 
\multicolumn{12}{@{}l@{}}{\textbf{With Proxy Model Training}} \\
\midrule 
DoReMi \citep{xie2023doremi}      & 37.3 & \sethlcolor{green!20}\hl{13.1} & \sethlcolor{green!20}\hl{60.4}   & \sethlcolor{green!20}\hl{28.4} & 46.3 & \sethlcolor{green!20}\hl{17.1} & \sethlcolor{red!20}\hl{29.8} & \sethlcolor{red!20}\hl{23.0} & \sethlcolor{green!20}\hl{10.8} & \sethlcolor{red!20}\hl{26.7}  & \textbf{32.4}  \\
\bottomrule
\end{tabular}
\caption{SacreBLEU scores (higher is better) on a chosen subset of OPUS-100 with a mixture of high (1M), mid (500K - 1M), and low ($<$500K) resource languages. The best performance is \textbf{bolded}. Scores that are close (within 1 BLEU) of the best performance are colored in \sethlcolor{green!20}\hl{green}. Scores lower than the best for more than 1.5 BLEU are highlighted in \sethlcolor{red!20}\hl{red}. \textbf{\sethlcolor{LightCyan}\hl{\name{}} outperforms various static and dynamic Temperature Sampling methods,} by improving the performance on LRLs without sacrificing much performance on HRLs.}
\label{table:mt}
\end{table*}

\paragraph{Machine Translation} Table \ref{table:mt} shows applying \name~on training a multilingual machine translation model. Compared to proportional sampling ($\tau=1$), \name~is able to greatly improve the mid and low-resource languages (+0.7 and 3.1 BLEU) while minimally sacrificing the performance of high-resource languages (-0.2 BLEU). Compared to static up-sampling $\tau = 5$, \name~matches the performance of mid- and low-resource languages while improving the performance of high-resource languages by 1.5 BLEU. Our method also outperforms the performance of \textit{Unimax} and \textit{Order Matters} scheduling while being easier to implement. Furthermore, \name~is able to match the performance of \textit{DoReMi} without having to train multiple proxy models. 

\paragraph{Multilingual Language Modeling} We also experimented with the general language modeling task on multiple languages on selected languages on the multilingual C4 (mC4) dataset \citep{xue-etal-2021-mt5}. We report the validation loss in table \ref{table:mc4_results}. Our results echo the findings in our machine translation experiments: \name~matches the performance on the only HRL English (EN) but outperforms other baselines on all three other languages.

\begin{table*}[]
\setlength\tabcolsep{9pt}
\centering
\label{tab:perplexity}
\begin{tabular}{@{}lcccc@{}}
\toprule
\textbf{Method} & \textbf{EN 2733B} & \textbf{IT 162B} & \textbf{ZH 39B} & \textbf{SW 1B} \\ 
\midrule 
\multicolumn{5}{@{}l@{}}{\textbf{Static Temperature Sampling}} \\ \midrule
$\tau=1$ (Proportional) & \textbf{2.67} & 3.10 & 4.23 & 4.09 \\ 
$\tau=5$ & 2.91 & 3.01 & 3.14 & 3.12 \\
$\tau=100$ ($\sim$Uniform) & 3.02 & 2.88 & 2.76 & 3.01 \\ \midrule
\multicolumn{5}{@{}l@{}}{\textbf{Dynamic Temperature Sampling}} \\ \midrule
Order Matters \citep{choi2023order} & 2.77 & 2.76 & 3.23 & 3.26 \\ 
Unimax \citep{chung2023unimax} & 2.98 & 2.91 & 2.85 & 3.06 \\ 
\sethlcolor{LightCyan}\hl{\name~(Ours)} & 2.75 & \textbf{2.63} & 2.56 & 2.94 \\ 
\midrule 
\multicolumn{5}{@{}l@{}}{\textbf{With Proxy Model Training}} \\ \midrule 
DoReMi \citep{xie2023doremi} & 2.89 & \textbf{2.63} & \textbf{2.51} & \textbf{2.89} \\ 
\bottomrule 
\end{tabular}
\caption{Dev loss on selected mC4 subset (lower is better). \textbf{\sethlcolor{LightCyan}\hl{\name}~achieves the best performance on all three LRLs (IT, ZH, and SW).} On the only HRL English, \name~is only behind proportional sampling ($\tau$=1), which is heavily optimized towards performance on English by sacrificing performance on all other languages.}
\label{table:mc4_results}
\end{table*}


\subsection{Study on Temperature Schedules}
\label{subsec:findings:ablations}

In this section, we revisit two design choices in dynamic Temperature Sampling: 
1) Increasing v.s. Decreasing the temperature during training, and 2) Dense v.s. Sparse updates of the temperatures. We highlight our results below:

\begin{figure}[ht]
    \centering
    \includegraphics[width=0.5\textwidth,trim=0cm 0.3cm 0cm 0cm]{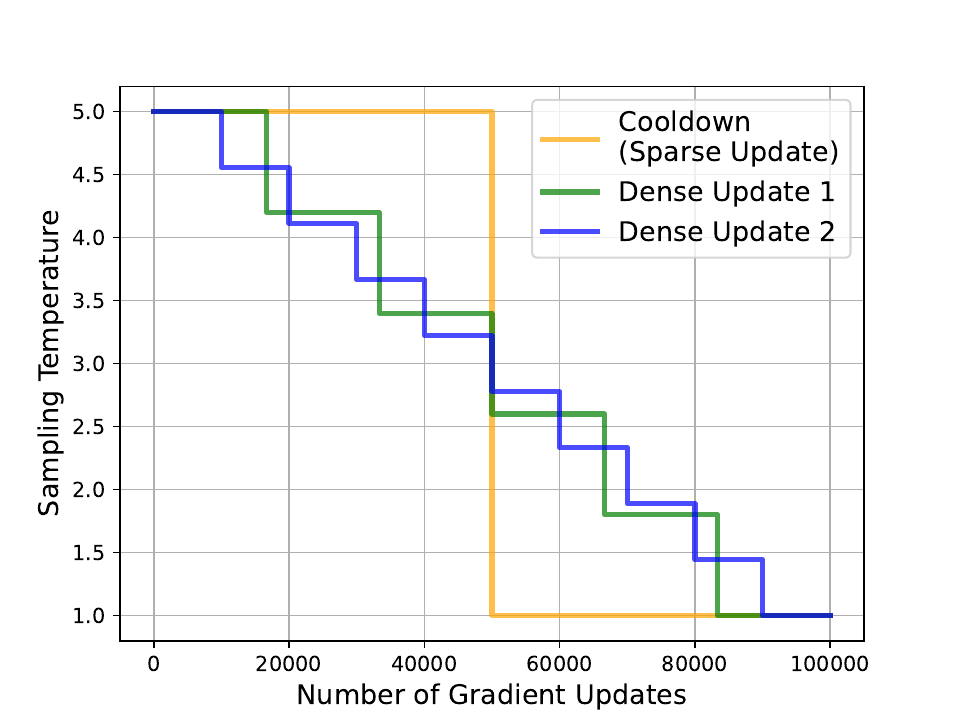}
    \caption{Sampling temperature schedules.}
    \label{fig:schedules}
    
    \vspace{0.5cm} 
    
    \setlength\tabcolsep{6pt} 
    \renewcommand{\arraystretch}{1.2} 
    \begin{tabular}{lcccc}
        \toprule
        Method & EN  & IT  & ZH  & SW  \\
        \midrule
        \rowcolor{orange!20} \textbf{\name}  & \textbf{2.59} & \textbf{2.63} & \textbf{2.56} & \textbf{2.94} \\
        \rowcolor{green!20} Dense Update 1 & 3.17 & 2.89 & 2.85 & 3.11 \\
        \rowcolor{blue!15} Dense Update 2 & 3.31 & 3.02 & 3.35 & 3.16 \\
        \bottomrule
    \end{tabular}
    
    \caption{Validation loss on mC4 subset (lower is better) for different update schedules in Figure \ref{fig:schedules}. Sparse updates perform better than dense updates.}
    \label{tab:mc4_dense_vs_sparse}
\end{figure}

\paragraph{Increasing the temperature is better than decreasing the temperature.} \textit{Unimax} \citep{chung2023unimax} upsamples the LRLs during the \textbf{beginning} of training while \textit{Order Matters} \citep{choi2023order} upsamples LRLs at the \textbf{end} of training. To resolve this conflict, we compare an increasing temperature schedule (1 for the first 15k training iterations and 5 for the second 15k training iterations; ``1-5") with a decreasing temperature schedule (5 for the first 15k iterations and 1 for the second half; ``5-1"), using the same \texttt{En-\{Zh, Ro\}} machine translation data in \S \ref{section3}. Figure~\ref{fig:enter-label} illustrates that a decreasing schedule (dashed) converges faster and results in better performance on the \sethlcolor{matplotlibblue!75}\hl{LRL} compared to an increasing schedule (solid), with minimal sacrifice on the \sethlcolor{matplotlibred!90}\hl{HRL}. This means that upsampling the LRLs during the beginning of training performs better.

\paragraph{Curse of Granularity} We concluded that a decreasing schedule generally leads to faster convergence and better overall performance. Furthermore, we compare various fine-grained decreasing schedules that perform dense decreasing of the sampling temperature with our sparse update (5 for the first 50k, 1 for the second 50k training iterations) using the subset in mC4 described in Table \ref{table:mc4_stats}. Figure~\ref{fig:schedules} illustrates the schedules we compare, and \ref{tab:mc4_dense_vs_sparse} shows the validation loss of each decreasing schedule. Fine-grained online reweighting method yields worse performance, echoing the findings in \citet{fan2024doge} and \citet{xie2023doremi}.\footnote{Both \citet{xie2023doremi} and \citet{fan2024doge} approximate sampling probabilities using Scalarization when performing dense updates, which could also be the reason why dense update underperforms in their experiments.}

\begin{figure}[htbp]
  \centering
\includegraphics[width=0.505\textwidth,trim=0cm 0.3cm 0cm 0cm]{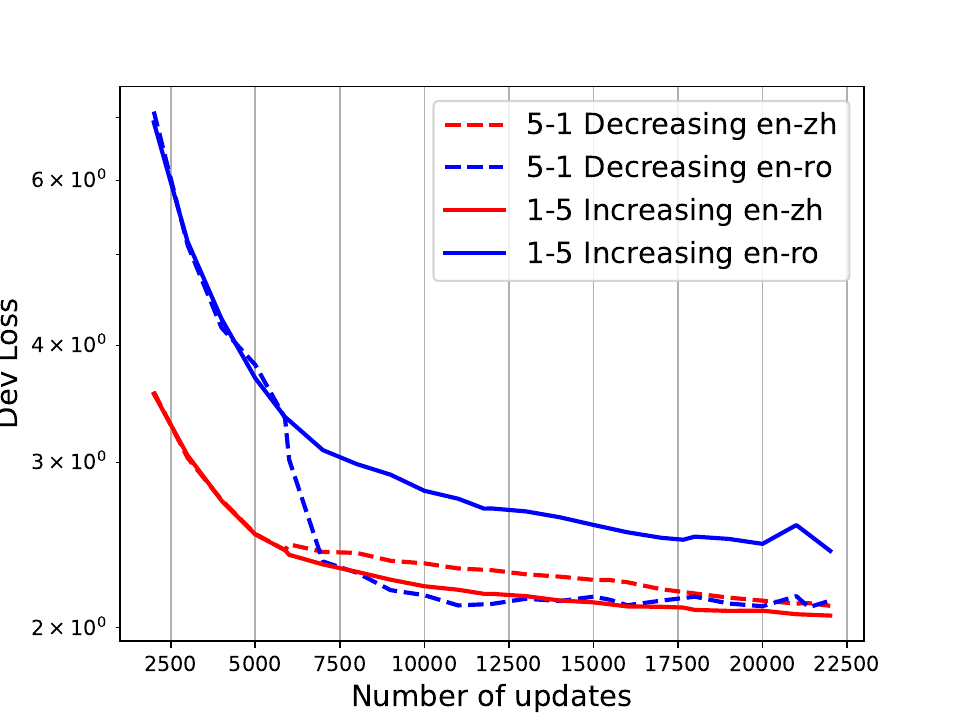}
  \caption{Comparison between an increasing (solid) Temperature Sampling schedule and a decreasing (dashed) schedule in multilingual machine translation \texttt{En-\{Zh, Ro\}}. \textbf{A decreasing temperature schedule outperforms an increasing one on the \sethlcolor{matplotlibblue!75}\hl{LRL} with minimal sacrifice on the \sethlcolor{matplotlibred!90}\hl{HRL}.}}
  \label{fig:increasing_vs_decreasing}
  \vspace{-15pt}
\end{figure}

\section{Related Works}


\paragraph{Gradient-based methods for Multi-Task Learning} Training a multi-domain language model can be seen as Multi-Task Learning (MTL; \citet{caruana1997mtl}), where each domain is a single task. Gradient-based methods aim to reduce the discrepancy in directions of conflicting gradients in different tasks: \textit{PCGrad} \citep{NEURIPS2020_3fe78a8a} aims to project the gradient of a task onto the orthogonal plane of the gradients of the other task. Another line of work \citep{10.5555/3524938.3525864, wang-etal-2020-balancing, kreutzer-etal-2021-bandits-dont} uses gradient similarity as the reward to train a policy that decides the sampling probabilities for each domain in a Reinforcement Learning setting. \citet{fan2024doge} utilizes gradient similarity between domains to design a temperature schedule for balancing multiple domains in training language models. However, recent studies point out that such gradient-based techniques do not yield significant improvement compared to a weighted sum of individual task losses (Scalarization) \citep{kurin2022in, NEURIPS2022_580c4ec4, NEURIPS2023_368559ed}. Our results also echo the findings of  \citet{zhai2023understanding}, where they find loss reweighting (Scalarization) underperforms standard training. Our work provides a possible explanation that Scalarization induces a larger variance in gradients.

\paragraph{Loss-based methods for Multi-Task Learning} Another line of work utilizes the loss instead of the gradients per task for optimizing MTL models. An intuitive method is to put more weight on the task with the highest loss. In statistical learning, Distributionally Robust Optimization (DRO) methods \citep{BenTal2011RobustSO, 10.1214/20-AOS2004, pmlr-v80-hashimoto18a, Sagawa*2020Distributionally} minimize the loss of the worst-performing subgroup to balance performance. \citet{oren-etal-2019-distributionally} and \citet{xie2023doremi} apply DRO to multi-domain language modeling to minimize the loss of a set of worse-performing domains. Similarly, \citet{zhou-etal-2021-distributionally} applies DRO to multilingual machine translation by minimizing the loss of a set of worse-performing translation directions. \name~can be seen as an efficient approximation of DRO methods by upsampling the worse-performing LRL and shifting the focus to the HRL once the LRL is sufficiently trained. Unlike DRO,  \name~do does not require training proxy models \citep{pmlr-v139-liu21f} and dense updates on domain weights. Our findings also connect to the fact that different languages act like regularizers in multi-task learning \cite{li-murray-2023-zero}.

We defer additional related works on addressing multilingual imbalance and the discussion between class imbalance and domain imbalance to Appendix \ref{appendix:related}.

\paragraph{Scaling Laws for Domain Mixture} The search space for the optimal domain weights at any given training iteration is combinatorically large. Existing works conduct comprehensive experiments on smaller scaled models to learn how the training and generalization error varies according to dataset sizes and domain weights --- ``scaling laws" of domain mixture \cite{ye2024data, ge2024bimixbivariatedatamixing, jiang2024adaptivedataoptimizationdynamic}. Closer to our work, \citet{chen2023on} fits scaling laws for sampling temperature for multilingual machine translation. Concurrent to our work, \citet{he2024scalinglawsmultilinguallanguage} fits scaling laws for multilingual language modeling. However, as \citet{jiang2024adaptivedataoptimizationdynamic} denotes: ``the optimal data policy for a smaller model
does not necessarily generalize to larger models.'' 

\section{Conclusion}
We examined two common balancing methods for multi-domain language modeling with data imbalances: Scalarization and Temperature Sampling. Although both yield the same loss, Temperature Sampling converges faster but risks overfitting. To mitigate this, we propose \name, a variant that adjusts temperatures during training to maintain fast convergence while reducing overfitting.

\section*{Limitations}
\label{section:limitations}
We discuss the limitations of our study here.

\paragraph{Impact of data mixture on downstream performance}  Studies have pointed out that the data mixtures in different domains impact downstream performance \citep{gururangan-etal-2020-dont, albalak2023efficient, fan2024doge}. Our work only focuses on the impact of different temperature schedules on pre-training validation performance. Although the effect of optimizing pre-training mixtures across different languages on downstream performances has not been fully concluded, works have shown that a lower pre-train validation loss generally leads to better downstream performance \citep{xie2023doremi, du2024understanding}. 
    
\paragraph{Difference between multi-lingual and multi-domain language modeling} Existing work on mono-lingual language modeling \citep{JMLR:v24:22-1144, pmlr-v162-du22c, xie2023doremi, oren-etal-2019-distributionally, fan2024doge, longpre2023pretrainers} maps all data from the same source (e.g. Wikipedia, Web, Books) to a single domain. Such a mapping ignores the subdomains within each source. In our work, we focused on the multilingual setup because (a) There exists a severe data size mismatch and (b) there is a clear and natural definition of ``domain" --- the different languages, and we expect the sampling to have a larger impact because of this mismatch. Even though we only conducted experiments on 
    a multilingual setup, our theoretical analysis applies to all setups with heavy dataset size mismatches. 
    
\paragraph{Finding the optimal temperature schedule} It requires a large amount of compute to thoroughly study the scaling laws of pre-training language models under different temperatures \citep{ye2024data} and to search for the optimal static sampling temperature $\tau$ for a given dataset \citep{chen2023on}, let along dynamic temperature scheduling. The optimal temperature depends not just on the size of the dataset but also on the ``difficulty" of the dataset.  Therefore, existing research \citep{chen2023on, xie2023doremi, oren-etal-2019-distributionally, zhou-etal-2021-distributionally, liu2024regmix, dubey2024llama3herdmodels} relies on training a proxy model on the dataset probe domain difficulty to determine optimal weights. Our work, instead, proposes a heuristic that decreases the temperature during training that does not rely on any proxy training, but we did not exhaustively test all the decreasing schedules. We leave finding optimal temperature schedules without preliminary training for future work.

\section*{Ethical Considerations}
One application of our study is to balance high- and low-resource languages. 
 We aim to mitigate biases that favor high-resource languages. However, this approach also raises risks, such as amplifying existing biases in limited and potentially skewed low-resource corpora. Furthermore, improved models could be misused to spread misinformation or infringe on privacy, especially in communities less equipped to counter such impacts. Thus, while \name~ promotes linguistic diversity, it requires careful monitoring to ensure it is used ethically.


\section*{Acknowledgements}
This work is supported by ONR grant (N00014-24-1-2089) and a gift from Allen Institute for AI. 
We are grateful to Nicholas Lourie and Jingyu Zhang for their insightful feedback throughout this project. 
We also thank the anonymous reviewers for their valuable feedback on our earlier draft. The GPUs were provided by the DSAI cluster.

\bibliography{custom}

\clearpage

\appendix 
\onecolumn

\begin{center}
{\Large \textbf{Supplemental Material}}
\end{center}

\section{Additional Related Works}
\label{appendix:related}

\paragraph{Multilingual Interference} Finding out the reasons and solutions for negative interference in Multilingual Neural Machine Translation \citep{johnson-etal-2017-googles, aharoni-etal-2019-massively} has been an active research area for the past decade. Yet, while previous studies \citep{wang2021gradient} find that negative interference mainly occurs between different language families, recent studies \citep{shaham-etal-2023-causes} have demonstrated that negative inference does not happen between languages of different families. The interference emerges because of the mismatch in the amount of data for different translation directions. Real-world translation data suffers from a heavy mismatch of data size in different directions, ranging from less than 100K to over 100M \citep{nllbteam2022language}. In our work, we show that this heavy mismatch in data size results in low-resource languages being under-trained. 

To mitigate interference caused by dataset sizes, \citet{aharoni-etal-2019-massively} and \citet{xue-etal-2021-mt5} propose to up-sample low-resource languages, which often results in the model overfitting on the LRLs while underfitting HRLs. \citet{huang-etal-2022-unifying} proposes to distill the model from earlier checkpoints with the LRLs that have not overfit with the current model to regularize the training of LRLs. \citet{huang-etal-2023-towards} proposes to distill between a model trained with a low sampling temperature and a model trained with a high sampling temperature. \textit{Unimax} \citep{chung2023unimax} proposes to first uniformly sample from all languages until an LRL dataset has been seen by the model for a fixed amount of repetitions; then, we remove the LRL from training. \textit{Order-Matters} \citep{choi2023order} proposes the opposite of \textit{Unimax} \citep{chung2023unimax}, to first only train on the HRL and add in the LRL after a fixed iteration. Our work shows that the \textit{Unimax} style of decreasing the temperature works better and proposes a simple alternative that does not require tracking how many times a model has seen an individual LRL dataset.

\paragraph{Class Imbalance v.s. Domain Imbalance} Class imbalance aims to address when the input $x$ is drawn from the same distribution $x \sim \mathcal{D}$ but the output labels $y$ are imbalanced. Domain imbalance, on the other hand, studies the problem when the input is drawn from different distributions with mismatched sizes $x \sim \{\mathcal{D}_1, \mathcal{D}_2, ... \mathcal{D}_{k}\}$, making no assumptions about the output labels $y$. Therefore, our study is distantly connected to adjusting the sampling probabilities to address class imbalance \cite{buda2017imbalance}. We refer readers to \citet{henning-etal-2023-survey} for a comprehensive survey of class imbalance in natural language processing.

\section{Proof of Theorem 2}
\label{proof_of_theorem}

\textbf{Theorem 2~}(Scalarization induces larger variance under Stochastic Gradient Descent)
\textit{
 Using the same notation in Corollary 1.1, we have $\mathrm{Var}(\nabla \mathcal{L}_{S}(x; \mathbf{w}_{\tau})) \geq \mathrm{Var}(\nabla \mathcal{L}_{TS}(x; \tau))$.} \\
 
\noindent We first proof a required lemma:

\noindent \textbf{Lemma 2.1}
Let $\mathcal{D} = \{\mathcal{D}_1, ..., \mathcal{D}_K\}$, $\forall j \in \{1, ... ,K\}$, let $p(j, \tau) = \frac{\mathcal{D}_j^\frac{1}{\tau}}{\sum_{k} \mathcal{D}_k^\frac{1}{\tau}}$, then $\sum_{i=1}^K \frac{p(i ; \tau)^2}{p(i; 1)} \geq  1$
holds true when $\forall j, \mathcal{D}_j \geq 0$.

\begin{proof}
We substitute $x_j = \mathcal{D}_j^\frac{1}{\tau}$, then:
$$
\sum_{i=1}^K \frac{p(i; \tau)^2}{p(i; 1)} = \frac{\sum_{i} x_i^{\tau}}{(\sum_i x_i)^2}\cdot\left(\sum_i x_i^{2-\tau}\right) $$
By the Cauchy-Schwartz inequality, we have:
$$
\left(\sum_i x_i^{\tau/2}\cdot x_i^{({2-\tau})/2}\right)^2 \leq \left(\sum_i x_i^\tau\right)\left(\sum_i x_i^{2-\tau}\right),
$$
which simplifies to:
$$
\left(\sum_i x_i\right)^2 \leq \left(\sum_i x_i^\tau\right)\left(\sum_i x_i^{2-\tau}\right),
$$
which implies that:
$$
1 \leq \frac{\sum_{i} x_i^{\tau}}{(\sum_i x_i)^2}\cdot\left(\sum_i x_i^{2-\tau}\right) = \sum_{i=1}^K \frac{p(i; \tau)^2}{p(i; 1)}.
$$
\end{proof}

\noindent Armed with Lemma 2.1, we come back to the proof of Theorem 2.

\begin{proof} Since $\nabla \mathcal{L}_{S}(x; \mathbf{w}_\tau)$ and $\nabla \mathcal{L}_{TS}(x; \tau)$ are unbiased estimates of the total gradient, we only need to show that the expectation of the \emph{squared} gradient is larger for the stochastic gradient under Scalarization.

\begin{align*}
\underbrace{\mathop{\mathbb{E}}_{x \sim \mathcal{D}} [(w_{f(x)} \nabla \mathcal{L}(x))^2]}_\text{Scalarization} &= \sum_{i=1}^K p(i;1) w^2_{f(x)} \nabla \mathcal{L}^2(x) \\
\sum_{i=1}^K \frac{p(i ; \tau)^2}{p(i; 1)}\nabla \mathcal{L}^2(x) 
& \geq \nabla \mathcal{L}^2(x) \sum_{i=1}^K p(i; \tau) \nabla \mathcal{L}^2(x) \\ 
 &  = \underbrace{\mathop{\mathbb{E}}_{\mathcal{D}_i \sim p(i;\tau)} \left[ \mathop{\mathbb{E}}_{x \sim \mathcal{D}_i} [\nabla \mathcal{L}^2(x)] \right]}_\text{Temperature Sampling}.
\end{align*}
\end{proof}

\section{Proof of Theorem 3 and Construction of Figure 2}
\label{proof_of_theorem_3}

\textbf{Theorem 3~}(Scalarization induces larger variance when approximating higher temperatures)

\textit{The difference in variance $\Delta = \mathrm{Var}(\nabla \mathcal{L}_{S}(x)) - \mathrm{Var}(\nabla \mathcal{L}_{TS}(x))$ is non-decreasing when $\tau \geq 1$.} \\

\begin{proof}
From the proof of Theorem 2, we know that the difference in variance $\Delta$ can be quantified by $\frac{p(\mathcal{D}_i; \tau)^2}{p(\mathcal{D}_i; 1)}$.
We substitute $x_j = \mathcal{D}_j^\frac{1}{\tau}$, then:
$$
\sum_{i=1}^K \frac{p(i; \tau)^2}{p(i; 1)} = \frac{\sum_{i} x_i^{\tau}}{(\sum_i x_i)^2}\cdot\left(\sum_i x_i^{2-\tau}\right). 
$$
Let $F(\tau) = \frac{\sum_{i} x_i^{\tau}}{(\sum_i x_i)^2}\cdot\left(\sum_i x_i^{2-\tau}\right)$ be a function of $\tau$. Taking its  derivative with respect to $\tau$: 
\begin{align*}
\frac{d F(\tau)}{d \tau} =  F(\tau) \times  \left[\left(\sum_i \frac{x_i^\tau \log x_i}{\sum_j x_j^\tau}\right) - \left(\sum_i \frac{x_i^{2-\tau} \log x_i}{\sum_j x_j^{2-\tau}}\right)\right].
\end{align*}

\noindent Taking the derivative of the term $\sum_{i} \frac{x_i^\tau \log x_i}{\sum_j x_j^\tau}$ with respect to $\tau$, we get:

$$\frac{d}{d \tau}\sum_{i} \left(\frac{x_i^\tau \log x_i}{\sum_j x_j^\tau}\right) = \frac{\sum_{j} x_j^\tau \sum_{i} x_i^\tau (\log x_i)^2 - (\sum_i x_i^\tau \log x_i)^2}{(\sum_{i} x_i^\tau)^2},$$ which is the variance of $\log x_i$ under the probability distributions $\frac{x^\tau_i}{\sum_{i} x^\tau_i}$. Similarly, the derivative of $\sum_i \frac{x_i^{2-\tau} \log x_i}{\sum_j x_j^{2-\tau}}$ with respect to $\tau$ is the negative variance of $\log x_i$ under distribution $\frac{x_i^{2-\tau}}{\sum_i x_i^{2-\tau}}$. Since variances are always non-negative, we conclude that the following difference:

$$
\left[\left(\sum_i \frac{x_i^\tau \log x_i}{\sum_j x_j^\tau}\right) - \left(\sum_i \frac{x_i^{2-\tau} \log x_i}{\sum_j x_j^{2-\tau}}\right)\right],
$$
is always non-negative.
By Lemma 2.1, we know that $F(\tau) \geq 1$. Therefore, the derivative $\frac{d F(\tau)}{d \tau}$ is always non-negative when $\tau \geq 1$, meaning that $F(\tau)$ is non-decreasing when $\tau \geq 1$. \\

\noindent Furthermore, when $\tau$ is strictly larger than 1 and not all $\mathcal{D}$ are equal, $F(\tau)$ monotonically increases with $\tau$. Showing that approximating a larger temperature using Scalarization induces a larger variance than approximating smaller temperatures. 
\end{proof}

\paragraph{Construction of Figure 2} 
We plot the function $F(\tau) = \sum_i \frac{p(i; \tau)^2}{p(D_i; 1)}$ against $\tau$ by starting with a uniform distribution and progressively increasing its skewness to resemble Zipf distributions. Specifically, we generate distributions $D_i \propto \frac{1}{i^\alpha}$ for various exponents $\alpha$ (ranging from 0 to higher values), where $i$ denotes the rank of each element. For each distribution, we compute the normalized probabilities $p(i; \tau) = \frac{\mathcal|{D}_i|^{1/\tau}}{\sum_j |\mathcal{D}_j|^{1/\tau}}$ across a range of $\tau$ values. This approach allows us to analyze how increasing the skewness of the distribution influences the behavior of $F(\tau)$ as a function of $\tau$.
\clearpage

\section{Detailed Hyper-Parameters}

We provide a comprehensive list of the hyper-parameters we used in this appendix section: \S \ref{section3} - Table \ref{s3}.
\label{appendix:hparams}

\begin{table*}[ht]
\centering
\small
\begin{tabular}{cccc}
\toprule
\textbf{Hyper-parameter} & \textbf{Value} & \textbf{Hyper-parameter} & \textbf{Value} \\
\midrule
Arch & wmt\_en\_de\_big & Label smoothing & 0.1 \\ 
\hline
Optimizer & adam & Adam epsilon & 1e-06 \\
\hline
Adam betas & "(0.9, 0.98)" & Learning rate scheduler & inverse\_sqrt \\
\hline
Learning rate & 0.0005 & Warmup updates & 4000 \\
\hline
Validate interval updates & 1000 & Dropout & 0.1 \\
\hline
Attention dropout & 0.1 & Weight decay & 0.0 \\
\hline
Max tokens & 32768 & Update frequency & 8 \\
\hline
Max source positions & 256 & Max target positions & 256 \\
\bottomrule
\end{tabular}
\caption{Detailed Hyper-parameters for experiments in \S \ref{section3}. We use the fairseq \citep{ott2019fairseq} implementation.}
\label{s3}
\end{table*}

\begin{table*}[ht]
\centering
\small
\begin{tabular}{cccc}
\toprule
\textbf{Hyper-parameter} & \textbf{Value} & \textbf{Hyper-parameter} & \textbf{Value} \\
\midrule
Arch & \textcolor{red}{iwslt\_de\_en} & Label smoothing & 0.1 \\
\hline
Optimizer & adam & Adam epsilon & 1e-06 \\
\hline
Adam betas & "(0.9, 0.98)" & Learning rate scheduler & inverse\_sqrt \\
\hline
Learning rate & 0.0005 & Warmup updates & 4000 \\
\hline
Validate interval updates & 1000 & Dropout & 0.1 \\
\hline
Attention dropout & 0.1 & Weight decay & 0.0 \\
\hline
Max tokens & \textcolor{red}{16384} & Update frequency & \textcolor{red}{4} \\
\hline
Max source positions & 256 & Max target positions & 256 \\
\bottomrule
\end{tabular}
\caption{Detailed Hyper-parameters for Machine Translation experiments in \S \ref{section:experiments} on our selected subset of opus-100 \citep{zhang-etal-2020-improving}. We use the fairseq \citep{ott2019fairseq} implementation. Differences against Table \ref{s3} are in \textcolor{red}{red}.}
\label{s4}
\end{table*}

\clearpage

\section{Additional Results on Scalarization V.S. Temperature Sampling}
\label{zhroresults}
\begin{figure*}[h]
    \centering
    \includegraphics[scale=0.35]{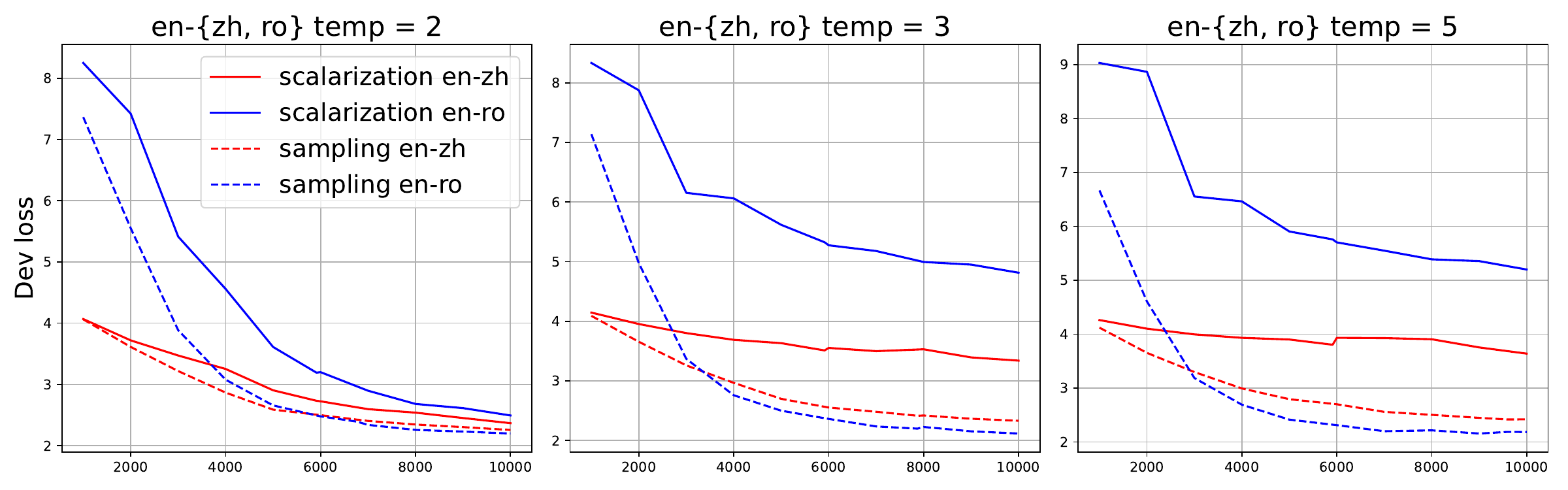}
    \caption{Validation loss by training iteration for \texttt{En-\{Zh, Ro\}}. Temperature Sampling (Dashed) converges much faster than Scalarization (Solid), especially at higher temperatures.}
    \label{fig:enter-label}
\end{figure*}

\begin{figure*}[htbp]
    \centering
    \begin{subfigure}{0.45\textwidth}
        \centering
        \includegraphics[width=\linewidth]{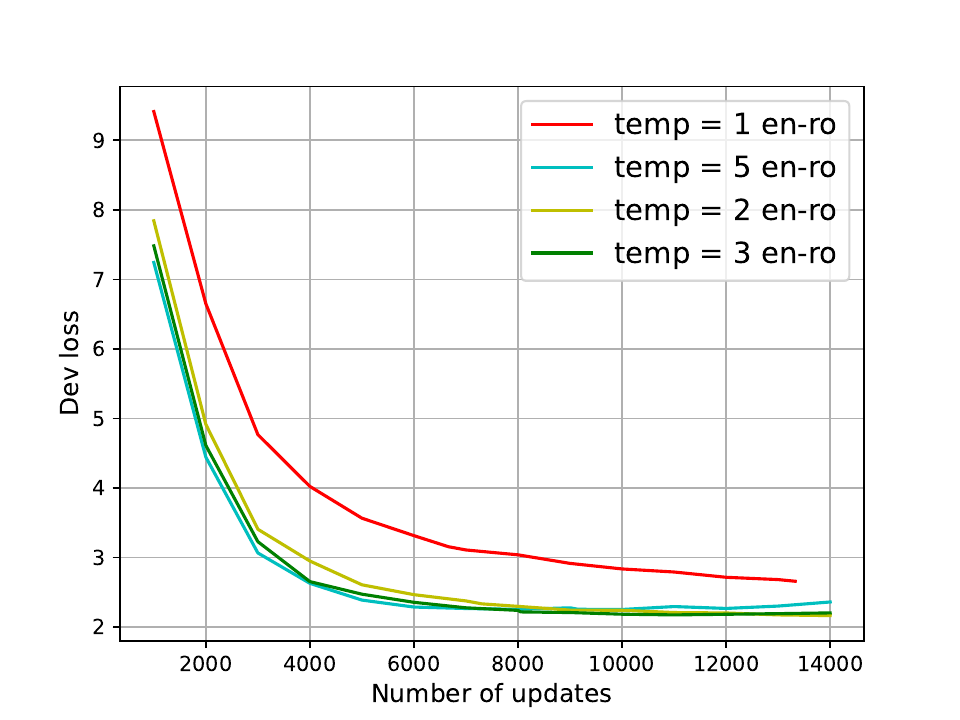}
        \caption{validation loss: low resource languages}
        \label{fig:lrl}
    \end{subfigure}%
    \begin{subfigure}{0.45\textwidth}
        \centering
    \includegraphics[width=\linewidth]{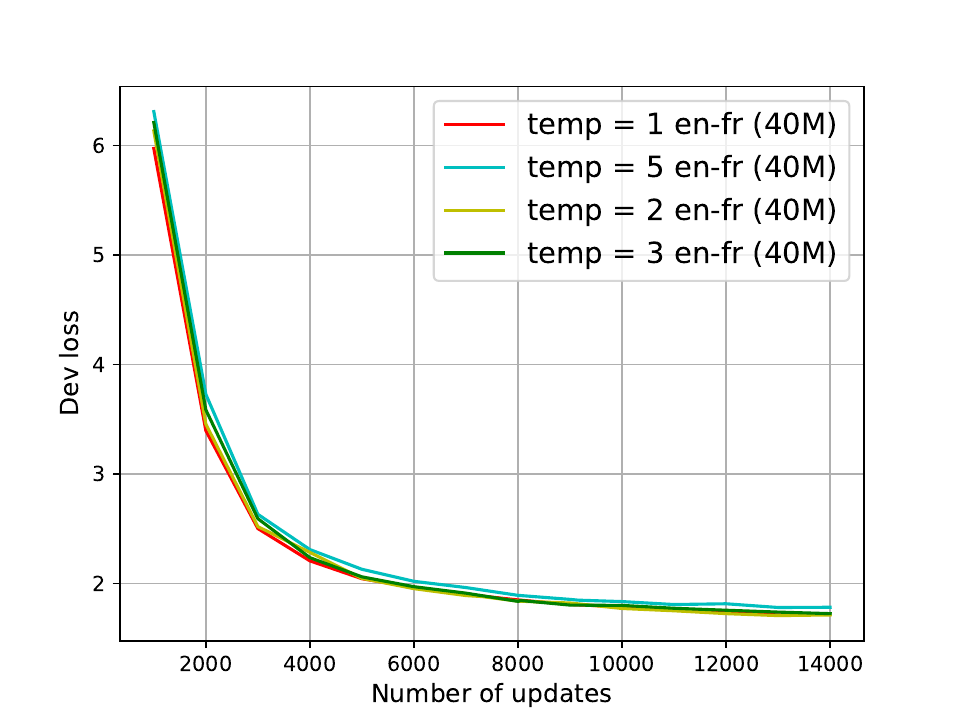}
    \caption{validation loss: high resource languages}
    \end{subfigure}
    \caption{Validation loss by gradient updates on the low-resource and high-resource language (left) jointly trained on the same model. Adjusting the sampling temperature has little impact on the high-resource language but a high impact on the low-resource language.}
    \label{fig:granularity}
\end{figure*}

\end{document}